\documentclass{article}
\usepackage{spconf,amsmath,graphicx}
\usepackage{amsfonts}
\usepackage{amssymb} %varnothing
\usepackage{color}
\usepackage{url}
\usepackage{enumitem}

\usepackage{algorithm}
\usepackage{algpseudocode}
\usepackage{amsthm} % theorem style
\usepackage{graphicx}
\usepackage{multicol}

\newtheoremstyle{defstyle} % name
    {0pt}                    % Space above
    {0pt}                    % Space below
    {\normalfont}            % Body font
    {}                       % Indent amount
    {\bfseries\itshape}      % Theorem head font
    {.}                      % Punctuation after theorem head
    {.5em}                   % Space after theorem head
    {\thmname{#1} \thmnumber{#2} \textbf{\thmnote{(#3)}}}    % Theorem head spec (can be left empty, meaning "normal")

\theoremstyle{defstyle}

\newtheorem{theorem}{Theorem}

\setlist[enumerate,1]{nosep, leftmargin=*}
\setlist[itemize,1]{nosep, leftmargin=*}

\newcommand\ui{\mathbf{u}}
\newcommand\DL{\mathbf{\Delta L}}

\title{Stability of Graph Convolutional Neural Networks \\through the lens of small perturbation analysis}
\name{Lucia Testa, Claudio Battiloro, Stefania Sardellitti, Sergio Barbarossa \thanks{This work was partially supported by  Huawei Technology France SASU, under agreement N. TC20220919044.}}
\address{DIET Department, Sapienza University of Rome, via Eudossiana 18, 00184, Rome, Italy \\
E-mail: \{lucia.testa, claudio.battiloro, stefania.sardellitti, sergio.barbarossa\}@uniroma1.it}
\begin{document}
\ninept
\maketitle

\begin{abstract}
In this work, we study the problem of stability of Graph Convolutional Neural Networks (GCNs) under random small perturbations in the underlying graph topology, i.e. under a limited number of insertions or deletions of edges. We derive a novel bound on the expected difference between the outputs of unperturbed and perturbed GCNs. The proposed bound explicitly depends on the magnitude of the perturbation of the eigenpairs of the Laplacian matrix, and the perturbation explicitly depends on which edges are inserted or deleted. Then, we provide a quantitative characterization of the effect of perturbing specific edges on the stability of the network. We leverage tools from small perturbation analysis to express the bounds in closed, albeit approximate, form,  in order to enhance interpretability of the results, without the need to compute any perturbed shift operator. Finally, we numerically evaluate the effectiveness of the proposed bound. 
\end{abstract}

\vspace{-.4cm}
\section{Introduction}
\vspace{-.2cm}
Over the recent years, deep learning (DL) has rapidly and extensively evolved, leading to significant advancements in various learning tasks and reaching state-of-the-art performance on a plethora of applications. Signals defined on irregular (non-Euclidean) domains, such as graphs, have become ubiquitous in modern machine learning and DL. Graph signals find utility across diverse domains like network science, recommender systems, cybersecurity, sensor networks, and natural language processing, just to name a few. Since their introduction \cite{scarselli2008graph,gori2005new}, Graph Neural Networks (GNNs) have demonstrated remarkable effectiveness in tasks that involve graph data. GNNs literature is extensive, and several methodologies have been explored, primarily divided into spectral \cite{Bruna19,kipf2016semi} and non-spectral techniques \cite{hamilton2017inductive, DuvenaudMABHAA15,GCNGama}. Essentially, the main idea is learning meaningful node embeddings or representations by aggregating information from neighboring nodes based on the graph's topology.  One of the most famous GNN architectures are Graph Convolutional Neural Networks, obtained by cascading layers made by banks of graph filters followed by pointwise non-linearities. A fundamental problem concerning GCNs is studying their stability under perturbations of the underlying graph topology \cite{gama,Keriven,Levie}, i.e. deletions or insertions of edges. Results about the stability of GCNs are essential in order to design robust GCNs, but also to gain important theoretical insights about why and how GCNs work. In particular, stability of GCNs has been shown to be directly related to their expressive/discriminative power \cite{gama} and to their transferability properties, i.e. the reusability of extracted features and the adaptability of the architecture on unseen graphs \cite{transf_spec_graph}, with possibly a significant different number of nodes and edges \cite{graphon_ruiz_transf}. In this paper, we are interested in studying the stability of GCNs under small perturbations, i.e. under a limited number of insertions or deletions of edges. The small perturbation regime is also practically encountered in several applications, such as wireless communication networks \cite{ceci2020graph}, social networks \cite{meng2023perturbation},  and biological networks.

\noindent\textbf{Related works.} The stability of GCNs and linear graph filters under perturbations has been studied from various perspectives. The seminal work in \cite{gama} related the stability problem to the expressive/discriminative power of GCNs, showing that the main advantage of GCNs relative to linear graph filters is their stability to graph perturbations. The work in \cite{Robust_Graph_Lipschitz_Constraints} proposed to constrain the frequency response of the GCN’s filter banks in order to obtain stable GCNs. Similarly, the work in \cite{Learning_Stable_Graph_Neural_Networks_via_Spectral_Regularization} developed a self-regularized GCN architecture that improves stability by regularizing the filter frequency responses. In \cite{cerv2022}, the authors propose another constrained learning approach to ensure the stability of GCN within a perturbation of choice.  In \cite{kenlay2021interpretable}, the authors studied the stability of linear graph filters with a focus on interpretability. The work in \cite{transf_spec_graph} studied the stability of graph filters belonging to the Cayley smoothness space, interestingly relating stability properties to transferability properties. In \cite{stabil_low_pass}, the authors studied the stability of low-pass linear graph filters under a large number of edge rewires, showing that the stability of the filter depends on perturbation to the community structure. 

\noindent\textbf{Contribution.} In this work, we study the stability of GCNs with the graph Laplacian as shift operator under a limited number of random insertions or deletions of edges, and we derive a novel bound on the expected difference between the outputs of unperturbed and perturbed GCNs. In this initial study,  our bound refers to a single-layer of the GCN, however our future work aims to extend it to a multi-layer networks.  Interestingly, the proposed bound is composed of two terms: a term proportional to the eigenvalues perturbations and a term proportional to the eigenvectors perturbations. Both terms are explicitly related to the specific edges that are perturbed, allowing us to quantitatively characterize the effect of the perturbations via spectral graph theory, thus enhancing interpretability and relate the stability of GCNs to perturbation of the underlying community structure, similarly to \cite{stabil_low_pass}. Moreover, we leverage tools from small perturbation analysis to express the bound in closed form,  computable by only knowing the perturbation probabilities, without the need to compute any perturbed shift operator and its eigendecomposition. We also derive a Hoeffding-like inequality to evaluate the probability that the difference between the outputs of unperturbed and perturbed GCNs may deviating more than a given value from its expected upper bound. The works in \cite{kenlay2021interpretable},\cite{stabil_low_pass},\cite{transf_spec_graph} derive bounds on the difference between the outputs of unperturbed and perturbed graph filters. Our approach differs from them in some crucial aspects. Beyond the main and clear difference of analyzing GCNs and not linear graph filters, this is the first work that leverages tools from small perturbations analysis to study the stability of GCNs. Moreover, none of the aforementioned works carries out a statistical analysis to provide expected bounds. The work in \cite{stabil_low_pass} provides a bound explicitly dependent on eigenvalues and eigenvectors perturbations, however under the assumption of low pass filters and a large number of edge rewires, and with no statistical analysis. We evaluate the correctness and tightness of our bound by comparing it with the benchmark GCN bound in \cite{gama}. Finally, we show how our bound is sensitive to changes in the topology in a synthetic graph classification task.

\section{Small Perturbation Analysis \\of Graph Laplacian}
In this section, we briefly recall the theory of small perturbation analysis of the graph Laplacian eigendecomposition developed in \cite{ceci2020graph}. Assume to have a graph $\mathcal{G} = \{\mathcal{N}, \mathcal{E}\}$ with associated graph Laplacian $\mathbf{L} \in \mathbb{R}^{N \times N}$, where $N = |\mathcal{N}|$ is the number of nodes. A small perturbation of $\mathcal{G}$ corresponds to insert or delete a few edges, thus to modify the edge set $\mathcal{E}$, obtaining a perturbed edge set $\widetilde{\mathcal{E}}$. The perturbed graph $\tilde{\mathcal{G}} = \{\mathcal{N}, \widetilde{\mathcal{E}}\}$ is described by a perturbed graph Laplacian $\widetilde{\mathbf{L}} = \mathbf{L} + \Delta\mathbf{L},$ where  $\Delta\mathbf{L} \in \mathbb{R}^{N \times N}$ is a perturbation matrix. Let us suppose that a certain (small) number of edges $\mathcal{E}_d \subset \mathcal{E}$ is deleted, and a certain (small) number of edges $\mathcal{E}_a$ is inserted. Let us denote with $\mathbf{a}_m \in \mathbb{R}^{N \times 1}$ the column vector whose entries are all zero except  the entries $\mathbf{a}_m(i_s)=1$ and $\mathbf{a}_m(i_t)=-1$ corresponding to the indices $i_s$ and $i_t$ of the endpoints of edge $m$. Then, the perturbation matrix is  
\begin{equation}
    \Delta\mathbf{L} = \sum_{m \in \mathcal{E}_p} \phi_m \mathbf{a}_m\mathbf{a}_m^T,
\end{equation}
where  $\mathcal{E}_p = \mathcal{E}_a \bigcup \mathcal{E}_d$ is the total set of perturbed edges, while $\phi_m=1$ if edge $m$ is added and $\phi_m=-1$, if edge $m$ is removed. In the case where all eigenvalues are distinct
and the perturbation affects a few percentage of edges, the perturbed eigenvectors $\{\widetilde{\mathbf{u}}_i\}_{i=1}^N$ and eigenvalues $\{\widetilde{\lambda}_i\}_{i=1}^N$ of $\widetilde{\mathbf{L}}$ are related to the unperturbed eigenvectors $\{\mathbf{u}_i\}_{i=1}^N$ and eigenvalues $\{\lambda_i\}_{i=1}^N$ of $\mathbf{L}$ by the following approximations \cite{ceci2020graph}:
\begin{align}
    &\widetilde{\lambda}_i \approx \lambda_i + \mathbf{u}_i^T\Delta\mathbf{L}\mathbf{u}_i \label{pert_eigenvalues}\\
    &\widetilde{\mathbf{u}}_i \approx \mathbf{u}_i + \sum_{j \neq i} \frac{\mathbf{u}_j^T\Delta\mathbf{L}\mathbf{u}_i}{\lambda_i - \lambda_j} \mathbf{u}_j\label{pert_eigenvectors}
\end{align}
Therefore, from \eqref{pert_eigenvalues}-\eqref{pert_eigenvectors}, the perturbations $\delta\lambda_i$ and $\delta\mathbf{u}_i$ of the $i$-th eigenvalue and eigenvector, respectively, can be written as \cite{ceci2020graph}:
\begin{equation}
\label{eq:variation}
    \begin{split}
    & \delta\lambda_i = \sum_{m \in \mathcal{E}_p} \phi_m\mathbf{u}_i^T\mathbf{a}_m\mathbf{a}_m^T\mathbf{u}_i  \\
    & \delta\mathbf{u}_i = \sum_{m \in \mathcal{E}_p} \phi_m \sum_{j \neq i} \frac{\mathbf{u}_j^T\mathbf{a}_m\mathbf{a}_m^T\mathbf{u}_i}{\lambda_i - \lambda_j} \mathbf{u}_j.
\end{split}
\end{equation}
The above formulas come from a first-order perturbation analysis \cite{ceci2020graph}, and they have been shown to be particularly useful in capturing the impact of the perturbation over the graph topology. In particular, it can be noted   that the 
term $\mathbf{u}_i^T \mathbf{a}_m \mathbf{a}_m^T \mathbf{u}_i= (u_i(v_{s,m})-u_i(v_{t,m}))^2$ in (\ref{eq:variation}) is a measure of the variation of the eigenvector $\mathbf{u}_i$ at the vertices $v_{s,m}$ and $v_{t,m}$ of edge $m$. This implies that the largest perturbations are observed over the  edges that have the highest eigenvectors variation. For example, in graphs with dense clusters, the edges with largest perturbations are the inter-cluster edges. Furthermore, we can observe that the eigenvector associated to
the null eigenvalue does not induce any perturbation on any other eigenvalue/eigenvector, and eigenvectors associated to eigenvalues very close to
nearby eigenvalues typically suffer larger perturbations. In this work, we will use the above formulas to derive a novel bound on the outputs of the layer of  Graph Convolutional Neural Networks (GCNs) under small perturbations, briefly reviewed in the next section.

\section{Graph Convolutional Neural Networks}
A Graph Convolutional Neural Network (GCN) is a deep neural architecture designed to process graph-structured data. The network processes an input graph signal $\mathbf{x} \in \mathbb{R}^{N}$, through a cascade of layers composed by: i) a polynomial graph filter (or filter bank) \cite{sandryhaila2013discrete} and ii) a pointwise non-linearity $\sigma(\cdot)$. In formulas, one layer of a GCN reads as:
\begin{equation}\label{gcn}
    \mathbf{y} = \sigma \Bigg( \mathbf{H}(\mathbf{L)}\mathbf{x}\Bigg) = \sigma \Bigg( \sum_{k = 0}^\infty h_k\mathbf{L}^k\mathbf{x}\Bigg)
\end{equation}
In general, the filter can be a polynomial of any shift operator. In this work, we consider the graph Laplacian $\mathbf{L}$. To make \eqref{gcn} practically implementable, the summation is generally truncated to a value $K<\infty$. In case a filter bank is employed, the output is a matrix collecting the output of each filter. As for any other learning model, the filter coefficients are generally learned by backpropagation in a data-driven and task-driven fashion. In the following section we derive our novel bound on GCNs outputs under small perturbations $\mathbf{y}$ considering a single-layer, single-filter GCN.

\section{Stability Bound}

In this section, we study the stability of GCNs by analyzing the changes in their output under perturbation of the input graph. We  measure stability using the following quantity:

\begin{equation}\label{distance}
    \mathbb{D}_\mathbf{H}(\mathbf{L},\mathbf{\tilde{L}};\mathbf{x}) = \left\lVert \sigma \Bigg( \mathbf{H}(\mathbf{L)}\mathbf{x}\Bigg) -\sigma \Bigg( \mathbf{H}(\mathbf{\tilde{L})}\mathbf{x}\Bigg) \right\rVert,
\end{equation}
where $\|\cdot\|$ is the $l_2$ norm. $\mathbb{D}_\mathbf{H}(\mathbf{L},\mathbf{\tilde{L}}; \mathbf{x})$ is then the distance between the unperturbed $\sigma \Big( \mathbf{H}(\mathbf{L)}\mathbf{x}\Big)$ and perturbed $\sigma \Big( \mathbf{H}(\mathbf{\tilde{L})}\mathbf{x}\Big)$ GCN outputs, for a given input $\mathbf{x}$.  Let us denote the frequency response \cite{sandryhaila2013discrete}  of $\mathbf{H}(\mathbf{L})$ with $h(\cdot)$. Our analysis hinges on the following mild assumptions:

\begin{enumerate}[label=\textbf{H\arabic*.}]
    \item\label{item:H0} \textbf{(Lipschitz filters)} We assume the frequency response $h(\cdot)$ to be Lipschitz so that, for all $a$, $b$ $\in \mathbb{R}$, it holds:
    \begin{align}
    | h(b) - h(a) | \leq C_L |b-a|,
    \end{align}
     where $C_L$ is the Lipschitz constant of the filter.
%    \item\label{item:H1} \textbf{(Differentiable Filters)} We assume the frequency response $h(\cdot)$ to be differentiable, so that, for all $a$ $\in \mathbb{R}$, it holds: 
 %   \begin{align}
 %   |h'(\lambda)| \leq C_D.
 %   \end{align}
    \item\label{item:H1} \textbf{(Lipschitz Nonlinearity)} We  assume the nonlinearity $\sigma(\cdot)$ to be Lipschitz so that, for all $a, b \in \mathbb{R}$, it holds:
    \begin{equation}
    | \sigma(b) - \sigma(a) | \leq C_\sigma| b - a |. 
    \end{equation}
    \item\label{item:H3} \textbf{(Small Perturbation.)} We assume that the graph perturbation affects only a small number of edges, i.e.:
    \begin{equation}
        |\mathcal{E}_p|<< |{\mathcal{E}}|.
    \end{equation}
\end{enumerate}
$\mathbf{H1}$ requires that the frequency response does not change faster than linear. If the analysis is limited to truncated filters, i.e. filters as in \eqref{gcn} but with the summation truncated to a value $K<\infty$, $\mathbf{H1}$ is naturally respected. $\mathbf{H2}$ holds true for practically every well-known nonlinearity (ReLU, TanH, Sigmoid,...). We are now able to present a first deterministic bound on the distance in \eqref{distance}.
\begin{theorem}
Let \(\sigma \Big( \mathbf{H}(\mathbf{\mathbf{L}})\mathbf{x}\Big)\) be a single layer GCN as in \eqref{gcn} and assume $\parallel \mathbf{x}\parallel=1$. Under assumptions \(\mathbf{H1},\mathbf{H2},\mathbf{H3}\), it holds:
\begin{equation}\label{bound_det}
    \mathbb{D}_\mathbf{H}(\mathbf{L},\mathbf{\tilde{L}};  \mathbf{x}) \leq
    \sum_{m\in {\mathcal{E}}_p} \Bigg( C \lVert \delta\mathbf{\Lambda}^{(m)}\rVert + C {N} \sqrt{\sum_{i=1}^{N} \lVert \mathbf{q}_{-i}^{(m)} \rVert^2} \Bigg) 
\end{equation}
where \(\delta\mathbf{\Lambda}^{(m)} = \text{diag}(\delta\lambda_{1,m}, \dots, \delta\lambda_{N,m}) \), \(\delta\lambda_{i,m}=\mathbf{u}_i^T \mathbf{a}_m \mathbf{a}_m^T\mathbf{u}_i\),  \(\sum_{i=1}^{N} \lVert \mathbf{q}_{-i}^{(m)} \rVert^2 = \sum_{i=1}^{N}\sum_{\substack{j=1 \\ j\neq i}}^N (\ui_j^T \DL^{(m)}\ui_i)^2\), denoting with $m$  the index of the perturbed edge, and $C = C_\sigma C_L$. 
\end{theorem}
\begin{proof} The proof is omitted due to space limitations, but it can be found in the supplementary material in \cite{supp_mat}.
\end{proof}
\noindent\textbf{Remarks.} %As the reader can notice from the proof of Theorem 1, the bound in \eqref{bound_det} holds true even if assumption $\mathbf{H4}$ is violated, i.e. even if the number of inserted or deleted edges is not small. However, 
Without assumption $\mathbf{H3}$, one should compute the perturbed Laplacian $\widetilde{\mathbf{L}}$ and its eigendecomposition to explicitly compute the bound. Assumption  $\mathbf{H3}$, thus the formulas in \eqref{pert_eigenvalues}-\eqref{pert_eigenvectors}, provide a low complexity way of computing the bound based on the eigendecomposition of only the unperturbed Laplacian $\mathbf{L}$. It is also worth to note that the bound in \eqref{bound_det} is independent of the input signal $\mathbf{x}$.
\section{Expected stability bound \\under random graph perturbations}

In this section, we provide a statistical analysis of the
bound in \eqref{bound_det}, assuming that the perturbation of the $m$-th edge is a random event happening with probability $p_m$. 
Building on the analysis of \cite{ceci2020graph} and our result presented in the previous section, we derive a bound on the expected value of  $\mathbb{D}_\mathbf{H}(\mathbf{L},\mathbf{\tilde{L}}, \mathbf{x})$\footnote{We drop the dependence of the bound on $\mathbf{x}$ because of \eqref{bound_det}.}.

\begin{theorem}
Let $p_m$ be the probability of the edge $m$-th to be perturbed. Under assumptions $\mathbf{H1},\mathbf{H2},\mathbf{H3}$, it holds: 
    \begin{equation}\label{bound_stoc}
        \mathbb{E}(\mathbb{D}_\mathbf{H}(\mathbf{L},\mathbf{\tilde{L}})) \leq \sum_{m\in\mathcal{E}_p} p_m\Bigg( C \lVert \delta\mathbf{\Lambda}^{(m)}\rVert + C {N} \sqrt{\sum_{i=1}^{N} \lVert \mathbf{q}_{-i}^{(m)} \rVert^2}\Bigg),    \end{equation}
where the expected value is taken w.r.t. the perturbation probability distribution and assuming  $\parallel \mathbf{x}\parallel=1$.
\label{th:th2}
\end{theorem}
\begin{proof} The proof is omitted due to space limitations, but it can be found in the supplementary material in \cite{supp_mat}.
\end{proof}

To further enrich our theoretical analysis, we derive a Hoeffding-like inequality for the bounds in \eqref{bound_det}-\eqref{bound_stoc}.
In particular, we are interested in evaluating the probability that $\mathbb{D}_\mathbf{H}(\mathbf{L},\mathbf{\tilde{L}})$
might deviate from its expected bound in \eqref{bound_stoc} by more than a given value $t$. For the sake of the exposition, let us denote the deterministic bound in \eqref{bound_det} with  

\begin{equation}
    B = \sum_{m\in\mathcal{E}_p} p_m\Bigg( C \lVert \delta\mathbf{\Lambda}^{(m)}\rVert + C {N} \sqrt{\sum_{i=1}^{N} \lVert \mathbf{q}_{-i}^{(m)} \rVert^2}\Bigg).
    \label{B_Bound}
\end{equation}

\begin{theorem}
Let $B$ be the quantity defined in \eqref{B_Bound} and let the assumptions of Th. $3.4$ from \cite{wainwright2019high}  to hold. Then,   
\begin{align}
    &P\{\mathbb{D}_\mathbf{H}(\mathbf{L},\mathbf{\tilde{L}},\mathbf{x}) > B + t\} \nonumber \\
    &
    \leq P\{\mathbb{D}_\mathbf{H}(\mathbf{L},\mathbf{\tilde{L}},\mathbf{x})>
    \mathbb{E}(\mathbb{D}_\mathbf{H}(\mathbf{L},\mathbf{\tilde{L}},\mathbf{x})) + t\}
    \leq \exp \left\{\frac{-t^2}{4B^2}\right\}.
\end{align}

\end{theorem}

\begin{proof}
The proof can be directly derived by Th. 3.4 and example 3.6 from \cite{wainwright2019high}.
\end{proof}
Given a positive value $\epsilon\leq  1$ such that $\exp{\left\{\frac{-t^2}{4B^2}\right\}} \leq \epsilon$ and consequently $t\geq \sqrt{4B^2\ln({\frac{1}{\epsilon}})}$. Then we can consider  $t$ as an  offset value from the mean such that $\mathbb{D}_\mathbf{H}(\mathbf{L},\mathbf{\tilde{L}})$ will not deviate from $\mathbb{E}(\mathbb{D}_\mathbf{H}(\mathbf{L},\mathbf{\tilde{L}},\mathbf{x})) + t$ with probability at least $1-\epsilon$.

% Following the $\mathbf{H3}$ we can say that \[\mathbb{D}_\mathbf{H}(\mathbf{L},\mathbf{\tilde{L}}) \leq \left\lVert \Big( \mathbf{H}(\mathbf{L)}\mathbf{x}\Big) -\Big( \mathbf{H}(\mathbf{\tilde{L})}\mathbf{x}\Big) \right\rVert\]
% From this we now write as $\mathbf{\Delta \H}=  \left\lVert \Big( \mathbf{H}(\mathbf{L)}\Big) -\Big( \mathbf{H}(\mathbf{\tilde{L})}\Big) \right\rVert $

\vspace{5mm} %5mm vertical space

\section{Numerical Results}
In our numerical exploration of the stability of Graph Convolutional Networks (GCNs) under topology uncertainties, we employ a systematic approach to generate, train, and perturb networks, subsequently analyzing the resulting effects. We quantitively show the effectiveness of the proposed bounds and their sensitivity to important perturbation of the graph topology. For every simulation we use  Frobenius norm.

\subsection{Deterministic stability bound}
We generated a synthetic dataset using the Stochastic Block Model (SBM), in order to perform a binary node classification task. We generated  50 graphs, two clusters per graph, with intra-clusters probability equal to 0.7 and inter-cluster probability equal to 0.08. Each node in the graphs is initialized by the SBM with a random value. Nodes with a feature value less than zero are categorized with label $0$, while nodes with a feature value greater than zero are categorized with label $1$. We employ a single-layer, single-filter GCN as in \eqref{gcn} with ReLU nonlinearity, and we separately train it on each graph to perform node classification. To test the proposed bound in the deterministic case, i.e. when we know which are the perturbed edges, we  remove an increasing number of edges from each of the generated graphs at inference time. We then compute the outputs 
 distance $ \mathbb{D}_\mathbf{H}(\mathbf{L},\mathbf{\tilde{L}}, \mathbf{x})$ from \eqref{distance} and its corresponding bound in (\ref{bound_det}), using the formulas in  \eqref{pert_eigenvalues}-\eqref{pert_eigenvectors}, per each graph.  
 We average the obtained distances and the corresponding bounds over the $50$ graph realizations and, for each graph,  over $1000$ random perturbations and feature realizations.
In Fig. \ref{fig:label1}, we report the averaged distances and the averaged bound against a maximum of 10 perturbed edges i.e. the 6\% of the total number of edges. As a comparison term, we report the stability bound from \cite{gama}.
It can be noticed that the bound increases with the number of perturbed edges, as expected, showing as a byproduct that our formulas in  \eqref{pert_eigenvalues}-\eqref{pert_eigenvectors} are very accurate in the small perturbation regime, i.e. when the number of perturbed edges is small. Observe also that our bound is significantly tighter than the bound derived in \cite{gama}, because it is able to better capture the eigenvalue/eigenvector perturbations.

\subsection{Expected stability bound}
To test the effectiveness of the expected bound in (\ref{bound_stoc}), we generated $50$ graphs. We tested on the same synthetic task,  using the same GCN architecture as in the previous section. We assume that all edges can potentially be perturbed with same probability $p_m = p > 0$. We vary the perturbation probability $p$ in a range from $0.01$ to $0.3$, but by maintaining the graph connected to avoid the change of multiplicity of the null eigenvalue. For each $p$, we compute the expected bound in \eqref{bound_stoc}, and we estimate the expected distance of the outputs $E(\mathbb{D}_\mathbf{H}(\mathbf{L},\mathbf{\tilde{L}},\mathbf{x}))$ in \eqref{bound_stoc} by computing and averaging $1000$ realizations of $\mathbb{D}_\mathbf{H}(\mathbf{L},\mathbf{\tilde{L}},\mathbf{x})$, each of them obtained perturbing edges with probability $p$. Finally, we averaged both the expected bounds and the estimated expected distances of the outputs over the $50$ graphs. From  Fig. \ref{fig:label2}, we can see that also in this case both the distance and its bound increase as the probability value grows.  

\subsection{Learning performance under perturbation}
Our goal in this section is to investigate  how the learning performance of a GCN is affected by graph topology perturbations and how the accuracy of the learning task depends on which edges are perturbed. For this purpose, we focus on a graph classification task in which the labelling process is based on the graph cut size. This implies that  the graph topology plays  a crucial role in the learning task \cite{leskovec2019}. 
%and how our bound is able to reflect this behavior. 
We designed a binary graph classification task  using a GCN. Specifically, we generate $300$ graphs through the SBM, where each graph is composed of $3$ communities with $10$ nodes each. To make sure that the GCN learns to classify based only on the graph topology, we set each node feature to be a constant value,  following a similar approach  as in \cite{leskovec2019}.
To label the graphs, we compute the cut size of each graph, using the median cut size as a pivotal threshold to assign a graph to a class or the other.
We employ the same GCN of the previous sections, and we train it to minimize the usual CrossEntropy loss using 200 graphs as a training set, and leaving out 100 graphs to be used as a test set. We are interested in investigating how the edge perturbations affect the learning accuracy considering that in this task perturbing certain edges instead of others can make a graph to change its class. 
%Therefore, we introduce controlled perturbations to the graphs belonging to the test set. 
%We analyzed four different scenarios. 
In Fig. \ref{fig:label3}, we show the classification accuracy of the GCN relative to the perturbation of up to 20 edges, i.e. the 12\% of the total edge count. In the plot we report four curves referring to the following perturbations: intra-cluster edge perturbation only (green); inter-cluster perturbation only (blue); mixed perturbation with different proportions of intra- and inter-cluster edge perturbation (yellow and purple). 
We can notice that the perturbation  of intra-cluster edges does not significantly affect the accuracy, because perturbing edges inside a cluster does not enforce the change of  class of the associated graph. On the other hand, as we increase the percentage of perturbed inter-cluster edges, the accuracy rapidly decreases, reaching its lowest values when all the pertubed edges are inter-cluster ones. 
In the plot, the red point indicates the accuracy calculated to the test set with no perturbed graphs and the patches on the plot indicate the values of the bounds relative to the distances of the  value of the corresponding accuracy.

% The patches on the plot indicate the values of the bounds relative to the distances of the outputs $\mathbb{D}_\mathbf{H}(\mathbf{L},\mathbf{\tilde{L}}, \mathbf{x})$ value of the corresponding accuracy.
% The reader can notice how the accuracy decreases very rapidly after few edges are perturbed, and how the proposed bound reflects these changes. 

\begin{figure}[t!]
        \centering
    \includegraphics[width=1\linewidth]{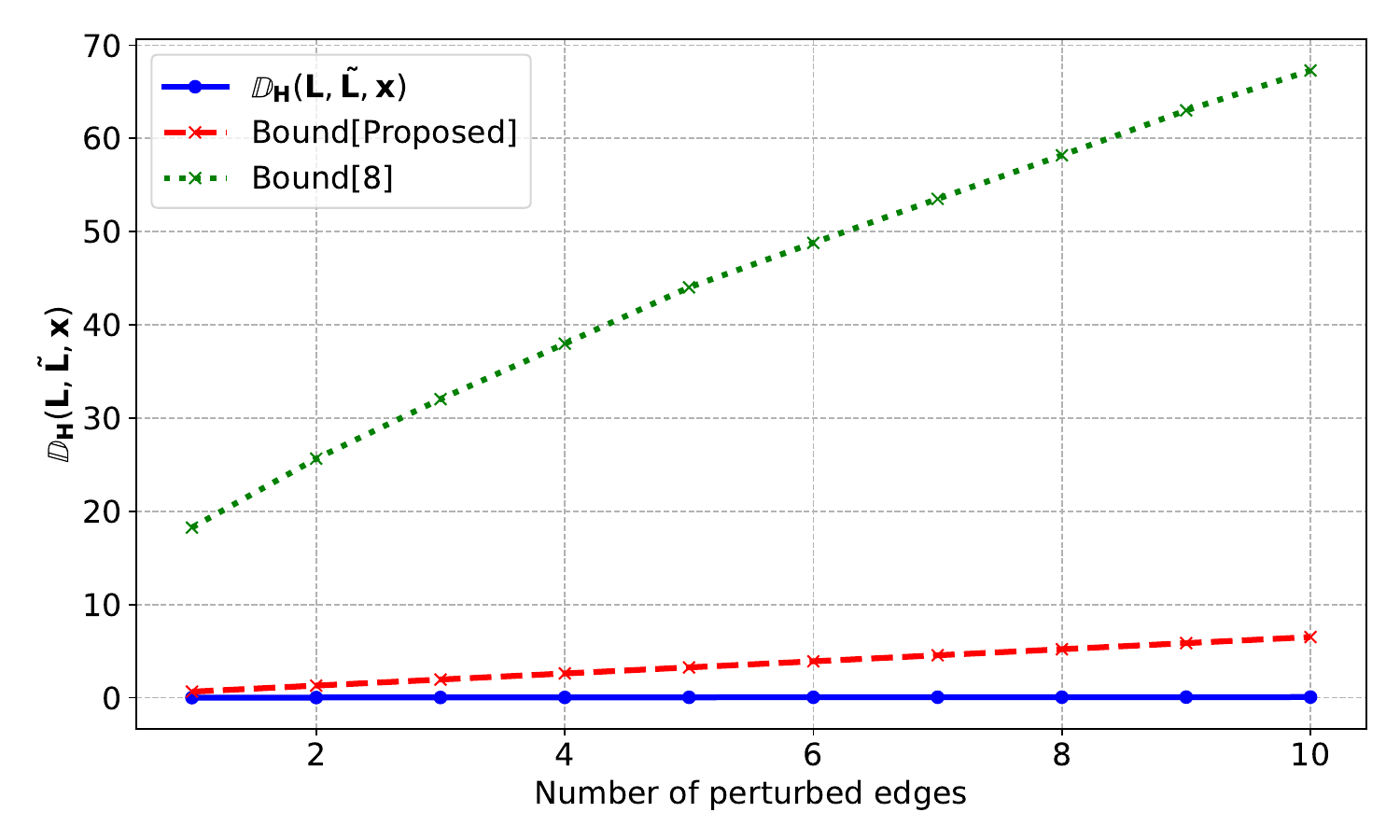}
        \caption{Distance of the outputs $\mathbb{D}_\mathbf{H}(\mathbf{L},\mathbf{\tilde{L}},\mathbf{x})$, vs  number of perturbed edges.}
        \label{fig:label1}
\end{figure}

% \begin{figure}[htb]

% \begin{minipage}[b]{1.0\linewidth}
%   \centering
%   \centerline{\includegraphics[width=8.5cm]{Bound analysis final plot.pdf}}
% %  \vspace{2.0cm}
%     \caption{Behaviour of the distance $\mathbb{D}_\mathbf{H}(\mathbf{L},\mathbf{\tilde{L}})$ and the two corresponding bounds(The one we proposed and the one in \cite{gama}}
%     \label{fig:label1}
 
% \end{minipage}
% %
% \begin{minipage}[b]{.48\linewidth}
%   \centering
%   \centerline{\includegraphics[width=4.0cm]{Probability Bound analysis.pdf}}
%   \caption{Behaviour of $\mathbb{E}(\mathbb{D}_\mathbf{H}(\mathbf{L},\mathbf{\tilde{L}}))$ and its corresponding bound presented in Theorem\ref{th:th2}}.
% \label{fig:label2}
% %  \vspace{1.5cm}

% \end{minipage}
% \hfill
% \end{figure}

\begin{figure}[t!]
        \centering
        \includegraphics[width=1\linewidth]{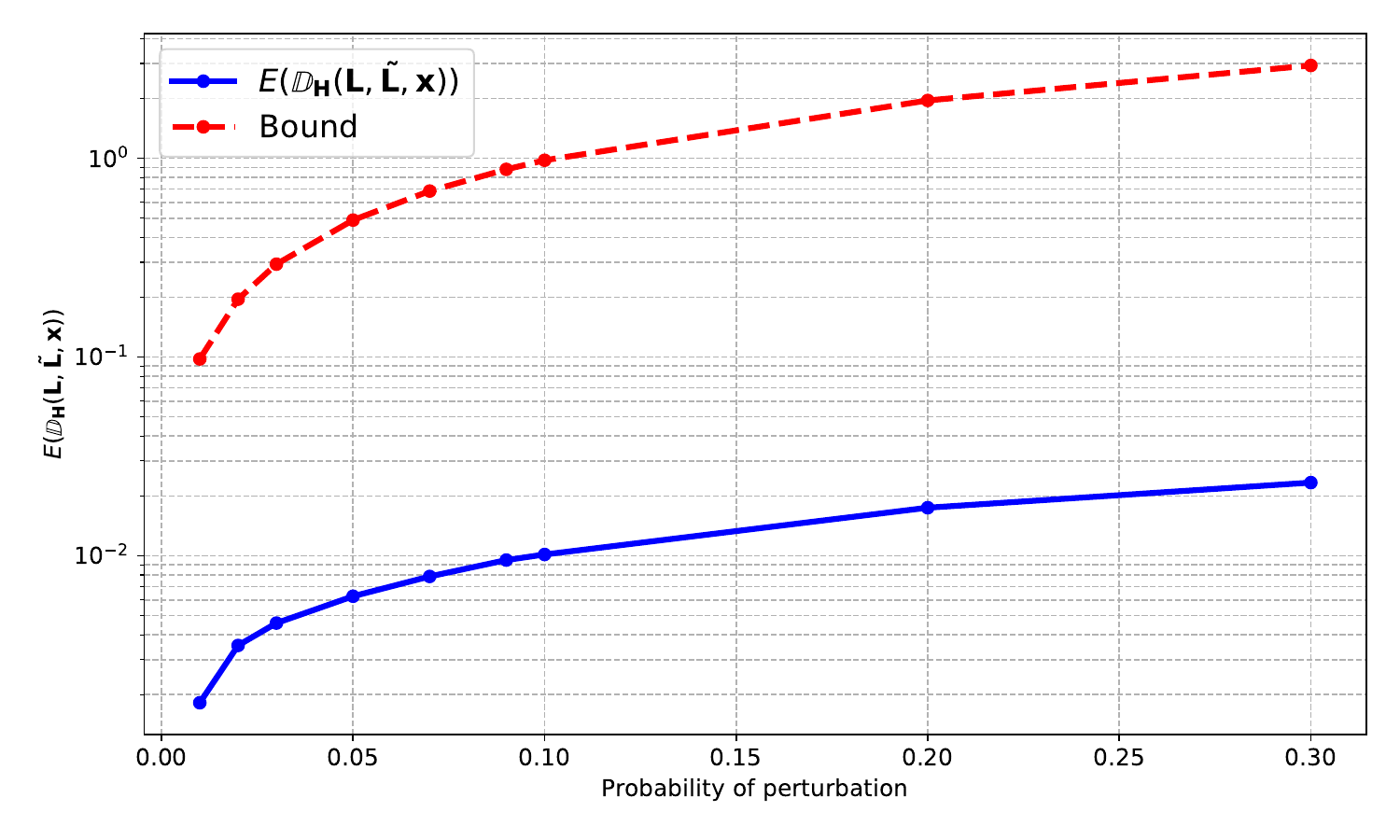 }
        \caption{Expected distance of the outputs $\mathbb{E}(\mathbb{D}_\mathbf{H}(\mathbf{L},\mathbf{\tilde{L}},\mathbf{x}))$ vs edge perturbation probability.}
        \label{fig:label2}
\end{figure}

% \begin{figure*}[htb!]
%     \centering
%     \begin{minipage}[b]{0.5\linewidth}
%         \includegraphics[width=\linewidth]{Bound analysis final plot.pdf}
%         \caption{Behaviour of the distance $\mathbb{D}_\mathbf{H}(\mathbf{L},\mathbf{\tilde{L}})$ and the two corresponding bounds(The one we proposed and the one in \cite{gama}}
%         \label{fig:label1}
%     \end{minipage}
%     \hfill
%     \begin{minipage}[b]{0.48\linewidth}
%         \includegraphics[width=\linewidth]{Probability Bound analysis.pdf}
%         \caption{Behaviour of $\mathbb{E}(\mathbb{D}_\mathbf{H}(\mathbf{L},\mathbf{\tilde{L}}))$ and its corresponding bound presented in Theorem\ref{th:th2}}.
%         \label{fig:label2}
%     \end{minipage}

% \end{figure*}

\begin{figure}[h!]
    \includegraphics[width=1\linewidth]{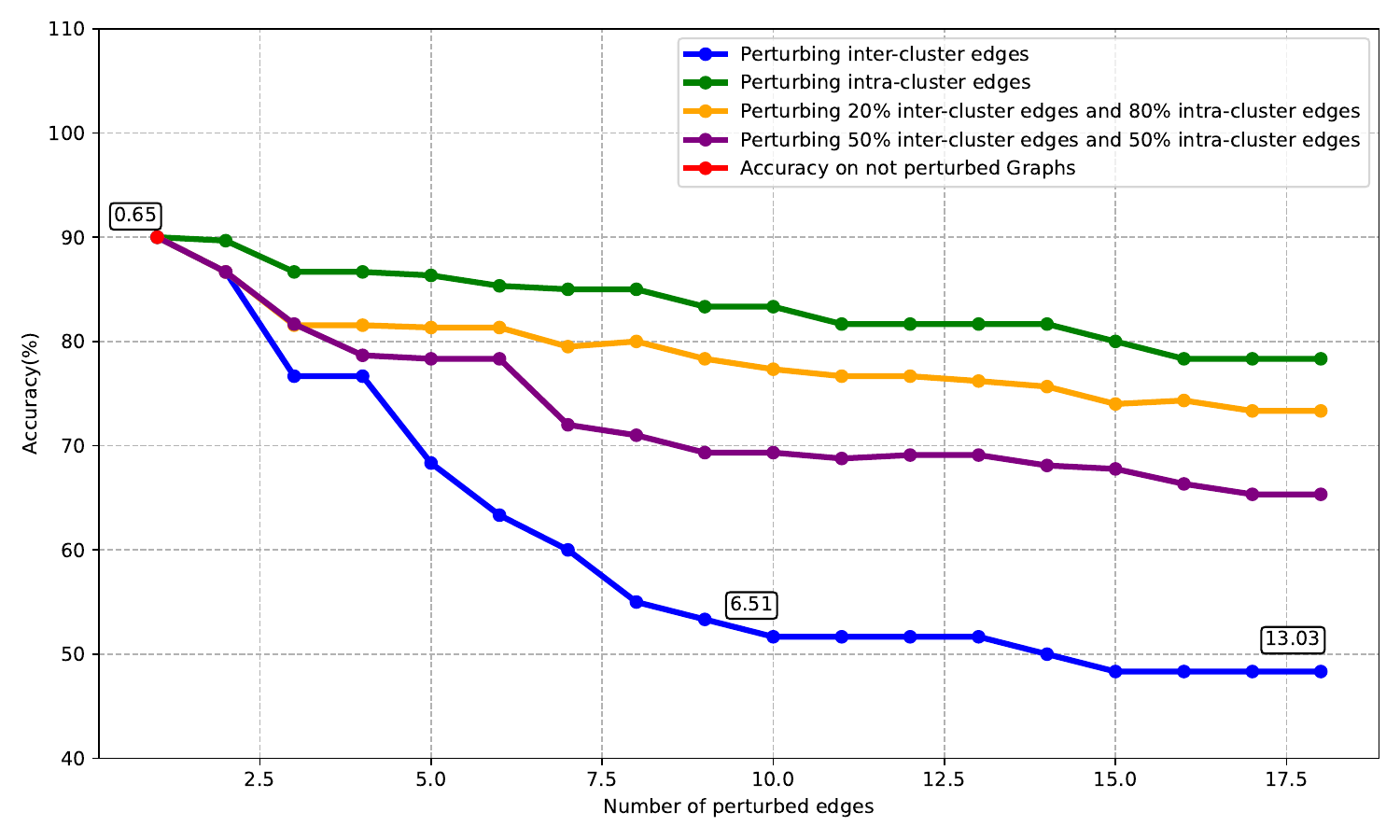}
        \caption{Classification accuracy vs  number of the perturbed edges.} 
        \label{fig:label3}
\end{figure}

\section{Conclusions}

In this work we introduced a novel way to measure the stability of a GCN under small perturbations analysis through the definition of a novel upperbound on the expected difference between the outputs of a GCN calculated on perturbed and unperturberd graphs. 
Through the tools of small perturbations analysis we made the bound computable without the necessity of computing a perturbed shift operator and its eigendecomposition. We numerically confirmed that the proposed bound increases with the number of perturbed edges. We also showed that its explicit dependence on the perturbations of the graph spectrum makes a tight bound. 
Finally, we showed how the classification error of a 1-layer GCN depends on which edges are perturbed and how the proposed bound reflects the changes in the learning performance. We plan to extend this work by considering a multilayer GCN in order to formulate a transferability theory for Graph Convolutional  Neural Networks using tools from small perturbation analysis.

\bibliographystyle{IEEEbib}
\bibliography{biblio.bib}
\end{document}